\pdfoutput=1

\documentclass[11pt]{article}

\usepackage[preprint]{acl}
\usepackage{times}
\usepackage{latexsym}
\usepackage{amsmath}
\usepackage{amsthm}
\newtheorem{theorem}{Theorem}
\usepackage{amssymb}
\usepackage{subcaption}
\usepackage{booktabs}
\usepackage{tabularx} 
\usepackage{ragged2e} 
\usepackage{siunitx} 
\usepackage{multirow}
\definecolor{myblue}{RGB}{111, 42, 221}
\definecolor{myred}{RGB}{204, 72, 158} 
\definecolor{mycolor1}{RGB}{254, 223, 223} 
\definecolor{mycolor2}{RGB}{215, 241, 217} 
\definecolor{mycolor3}{RGB}{255, 251, 195} 

\usepackage[T1]{fontenc}

\usepackage[utf8]{inputenc}

\usepackage{microtype}

\usepackage{inconsolata}

\usepackage[most]{tcolorbox}
\newtcolorbox{mybox}{colback=white!5!white,colframe=black!75!black, left=.05in, right=.05in}

\newtcbtheorem{exmp}{Example}%
{colback=gray!5,colframe=black,fonttitle=\bfseries, left=.02in, right=.02in,bottom=.02in, top=.02in}{ex}
\newtcbtheorem[no counter]{prompt}{}%
{colback=gray!5,colframe=black,fonttitle=\bfseries,separator sign none}{th}

\newtcbtheorem{ebox}{}
{colback=green2!5,colframe=blue1, left=.02in, right=.02in,bottom=.02in,title empty, top=.02in, theorem style=plain apart,frame empty}{ebox}

\usepackage{graphicx}
\usepackage[inline, shortlabels]{enumitem}
\title{DynamicMind: A Tri-Mode Thinking System for Large Language Models}

\author{
    \textbf{Wei Li\textsuperscript{1}}$^*$,
    \textbf{Yanbin Wei\textsuperscript{1,2}}\thanks{Equal contribution},
    \textbf{Qiushi Huang\textsuperscript{1,3}},
    \textbf{Jiangyue Yan\textsuperscript{1}},
    \textbf{Yang Chen\textsuperscript{1}},
    \\
    \textbf{James T. Kwok\textsuperscript{2}},
    \textbf{Yu Zhang\textsuperscript{1}}\thanks{\ \ Corresponding author} \\
    \\
    \textsuperscript{1}Southern University of Science and Technology
    \\
    \textsuperscript{2}Hong Kong University of Science and Technology
    \\
    \textsuperscript{3}University of Surrey
}

\begin{document}

\maketitle
\begin{abstract}
{Modern large language models (LLMs) often struggle to dynamically adapt their reasoning depth to varying task complexities, leading to suboptimal performance or inefficient resource utilization. To address this, we introduce \textbf{DynamicMind}, a novel tri-mode thinking system. DynamicMind empowers LLMs to autonomously select between Fast, Normal, and Slow thinking modes for zero-shot question answering (ZSQA) tasks through cognitive-inspired prompt engineering. Our framework's core innovations include: (1) expanding the established dual-process framework of fast and slow thinking into a  tri-mode thinking system involving a normal thinking mode to preserve the intrinsic capabilities of LLM; (2) proposing the \textbf{Thinking Density} metric, which aligns computational resource allocation with problem complexity; and (3) developing the \textbf{Thinking Mode Capacity (TMC)} dataset and a lightweight \textbf{Mind Router} to predict the optimal thinking mode. Extensive experiments across diverse mathematical, commonsense, and scientific QA benchmarks demonstrate that DynamicMind achieves superior ZSQA capabilities while establishing an effective trade-off between performance and computational efficiency. We will release our code and model checkpoint as soon as possible.
}

\end{abstract}
\section{Introduction}
\label{introduction}

In recent years, chain-of-thought (CoT) prompting techniques have significantly improved the zero-shot question-answering (ZSQA) capabilities of large language models (LLMs) by enabling step-by-step, deliberate reasoning to tackle complex tasks \cite{wei2022chain}. 
{Existing work \cite{guo2025deepseek} enhances LLMs' ability to solve complex problems, such as mathematical reasoning, by employing the \textit{slow thinking mode} through extensive intermediate reasoning processes based on CoT. However, some studies \cite{kojima2022large, NEURIPS2024_e304e04a} observe that for simpler tasks, like CommonsenseQA, LLMs perform better 
using the \textit{fast thinking mode}, which directly provides answers without intermediate reasoning steps, thereby significantly reducing computational costs, as exemplified in Figure \ref{fig:intro}. Those observations suggest the viability of employing dual-processing theory \cite{wason1974dual, tversky1974judgment}, 
drawing from human cognitive research on fast and slow thinking, to enhance the effectiveness and efficiency trade-off of LLM reasoning with the capabilities of fast and slow thinking.}

\begin{figure}[t]
    \centering
    \includegraphics[width=0.4\textwidth]{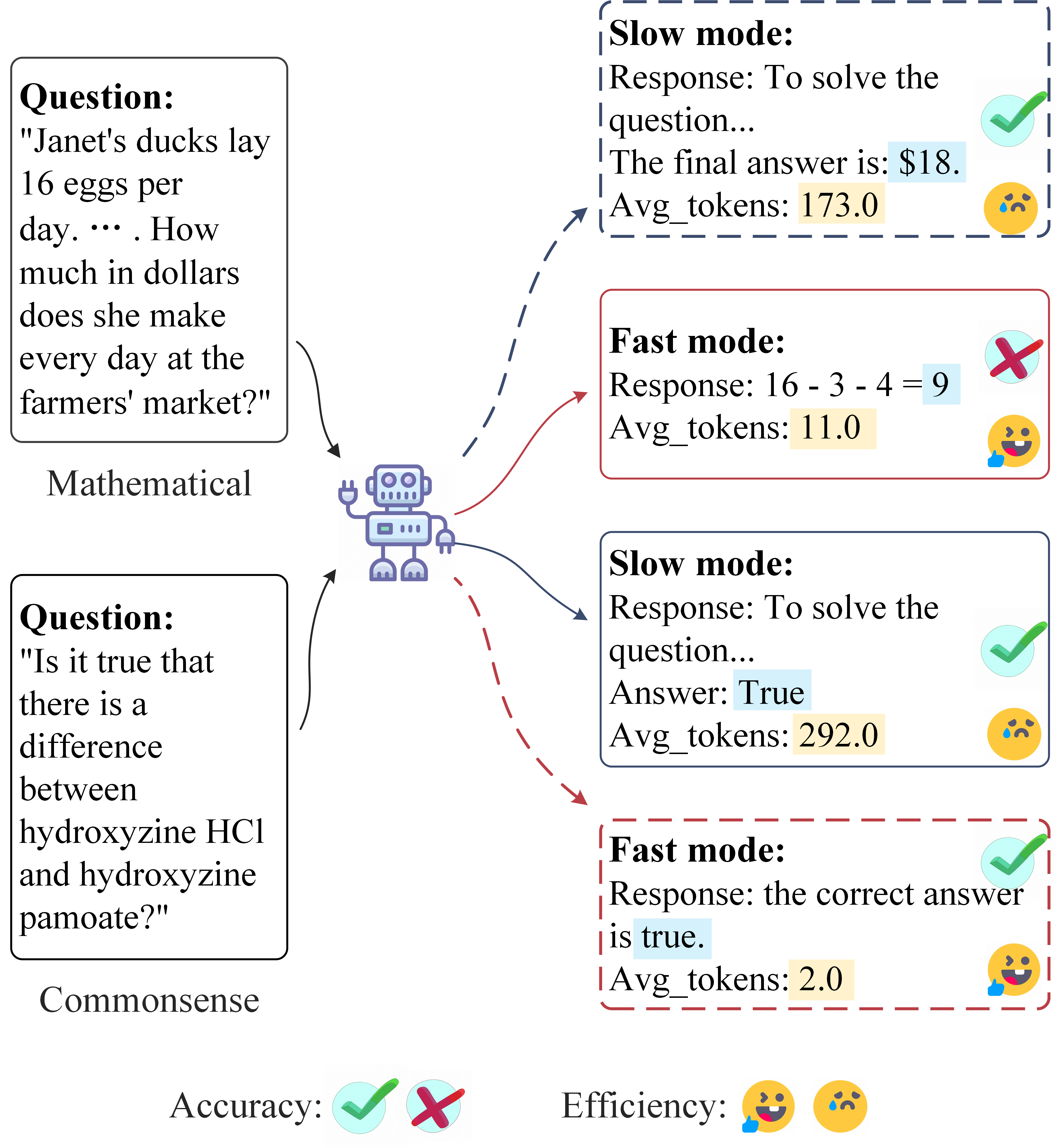}
\caption{Accuracy and efficiency trade-off for adopting fast or slow thinking mode in LLMs. The dashed line represents the ultimately preferred choice mode.
}
    \label{fig:intro}
    \vspace{-5pt}
\end{figure}

{Nevertheless, current dual-processing frameworks overlook the intrinsic reasoning capabilities acquired by LLMs during their original training, which we term the \textit{normal thinking mode}. Typically, LLMs are optimized to operate efficiently in their inherent normal mode, producing quick responses with moderate reasoning suitable for a broad range of tasks. For instance, in the GSM8K benchmark \citep{cobbe2021training}, we observe 80.52\% samples are more efficiently addressed by the normal mode than by either fast or slow modes. This is because LLMs already possess a foundational reasoning capability after their original training \cite{liu2024deepseek}, enabling them to solve a majority of problems without resorting to slow thinking. However, despite its general utility, the normal mode is not universally optimal. It can be less token-efficient than the fast mode for some simple tasks and may lack the profound reasoning capabilities of the slow mode when addressing highly complex problems.}

{LLMs can utilize three thinking modes with unique strengths. The fast mode is token-efficient for simple tasks like CommonsenseQA but may fail on complex problems such as math by favoring knowledge retrieval over computation. The slow mode ensures high accuracy for intricate, deep-reasoning tasks but incurs substantial token consumption, proving to be inefficient for simpler tasks. The normal mode, stemming from original training, seeks a balance between them, yet its inherent token consumption versus thinking depth trade-off is neither explicit nor optimally controlled. Our experimental results reveal that no single mode consistently optimizes both effectiveness and token efficiency across diverse tasks. Therefore, an adaptive system that can dynamically adjust thinking modes based on task complexity is crucial for enhancing the overall reasoning efficiency and versatility of LLMs.
}

To achieve that, {in this paper, we introduce \textbf{DynamicMind}, an adaptive framework extending dual-process theory into a tri-mode dynamic thinking system comprising \textit{fast, normal}, and \textit{slow} thinking modes for LLMs. The proposed DynamicMind system, guided by a lightweight \textbf{Mind Router}, dynamically adjusts its operational mode based on task complexity, and it features three key innovations. Firstly, we formalize the tri-mode thinking system, delineating the distinct impact of each mode on task performance. Secondly, we propose the \textbf{Thinking Density}, a mathematically grounded metric designed to quantify the trade-off between effectiveness and efficiency across these three thinking modes. Thirdly, we construct the \textbf{Thinking Mode Capacity (TMC)} dataset and develop the Mind Router, a lightweight module responsible of predicting the suitable thinking mode for LLMs.}
{Extensive empirical evaluations across 12 reasoning benchmarks, including mathematical, commonsense, and scientific tasks, demonstrate that DynamicMind effectively achieves an adaptive balance between the performance and efficiency. Furthermore, those advantages are shown to be able to generalize to unseen domains/tasks.} 

\begin{figure*}[htbp]
  \includegraphics[width=\textwidth]{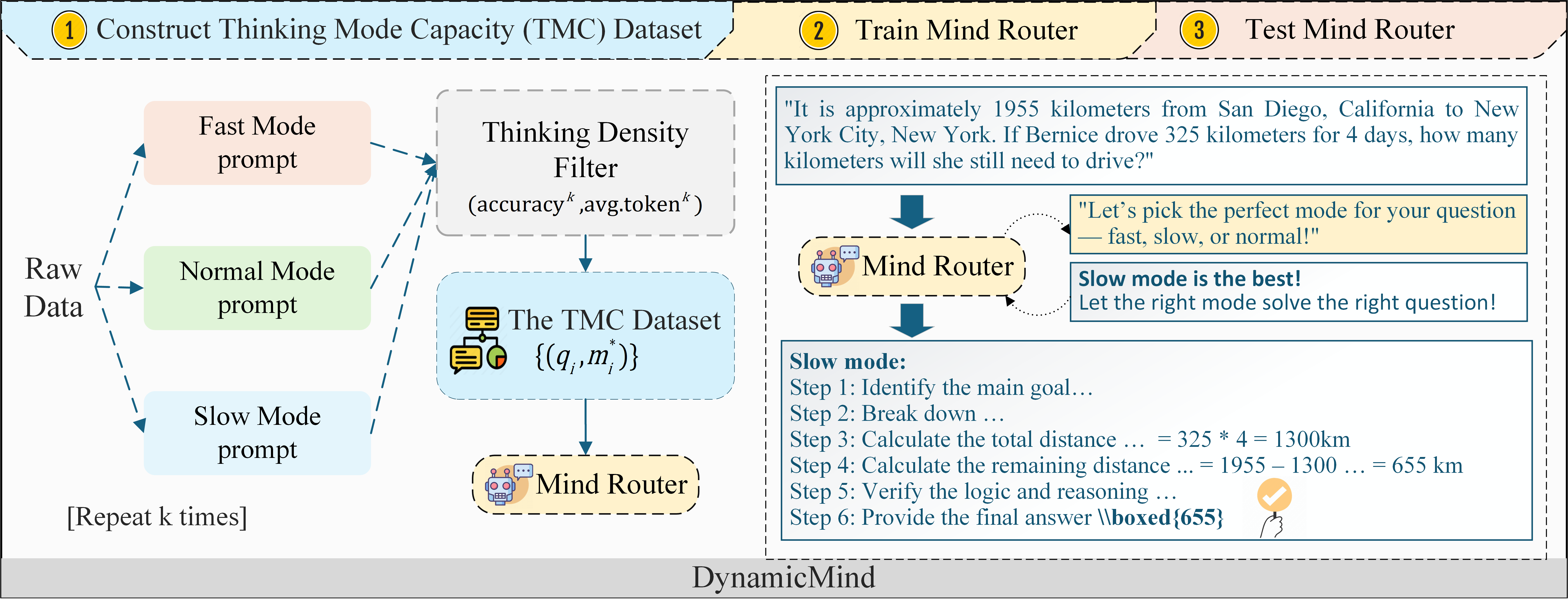}
  \caption{Overview of the DynamicMind framework, where the Mind Router guides the LLM Thinker to use the most appropriate thinking mode based on the question.}
  \label{fig:MD_framework}
  \vspace{-5pt}
\end{figure*}
\section{Related Work}
\label{related work}


\paragraph{Fast/Slow Thinking in Cognitive Science.}
The concept of dual-processing in human cognition, introduced by \citet{wason1974dual} and expanded by \citet{daniel2017thinking}, describes two systems: System 1 (fast, intuitive) and System 2 (slow, analytical). This framework has been influential in psychology and social sciences, providing insights into human decision-making \cite{frederick2005cognitive}.

\paragraph{Fast/Slow Thinking of LLMs.}

LLMs mimic human cognitive systems, with fast thinking aligning with System 1, excelling in quick, intuitive tasks but struggling with complex reasoning \citep{nye2021show}. Slow thinking mirrors System 2, using CoT reasoning to handle complex problems effectively \citep{chen2024not, qu2025survey}.
Existing methods mainly optimize slow thinking speed, such as compressed reasoning chains \citep{wang2023scott} and dynamic token skipping \citep{pan2024dynathink}, often neglecting fast thinking's potential. Other approaches like Dualformer \citep{su2024dualformer} and RL-based methods \citep{qu2025survey} attempt to configure both thinking modes. However, these approaches overlook the model's inherent task-based trade-off capabilities acquired during origin training. DynamicMind uses a lightweight, trainable Mind Router to dynamically switch a single LLM between fast, normal, and slow thinking modes, enhancing reasoning efficiency without extra training.

\paragraph{Efficient Thinking of LLMs.}

Dual-process thinking enhances cognitive efficiency. Fast thinking is advantageous for simple problems, avoiding overthinking, while slow thinking excels in complex tasks \citep{chen2024not}. Strategies to improve LLM reasoning efficiency include model-centric methods like CoT fine-tuning \citep{yu2024distilling} and RL frameworks \citep{shen2025dast}, but these are computationally intensive. Output-level strategies, such as Best-of-N and Majority Voting \citep{wu2025more}, enhance inference but increase computational costs. Input prompt methods instruct LLMs to operate with constrained token budgets \citep{han2024token}, though consistency challenges remain. Some approaches route queries to different LLMs based on difficulty \citep{ong2024routellm}, increasing resource demands.

\section{Methodology}
\label{methodology}

In this section, we present the proposed DynamicMind method.

\subsection{Problem Settings}
Zero-shot question answering (ZSQA) aims to develop a system capable of accurately answering questions without prior exposure to specific question-answer pairs. Unlike traditional QA systems that rely on extensive labeled datasets, ZSQA requires answering new questions based on the established linguistic understanding capabilities and parametric knowledge of models. This poses a significant challenge to the generalizability of existing LLM-based QA systems in handling unseen questions.

\begin{table}[!ht]
 \centering
 \footnotesize
 \caption{Characteristics of each thinking modes.}
 \label{tab:mode_spec}
 \begin{tabular}{cc}
 \toprule
 \textbf{Thinking Mode} & \textbf{Key Attributes} \\
 \midrule
 Fast Mode ($M_f$) & Intuitive, Limited Thinking Depth \\
 \midrule
 Normal Mode ($M_n$) & Native Capability, Balanced \\
 \midrule
 Slow Mode ($M_s$) & Analytical, Low Efficiency \\
 \bottomrule
 \end{tabular}
 \vspace{-15pt}
\end{table}

\begin{figure*}[!ht]
\begin{minipage}[b]{0.36\linewidth}
\centering
\includegraphics[width=\linewidth]{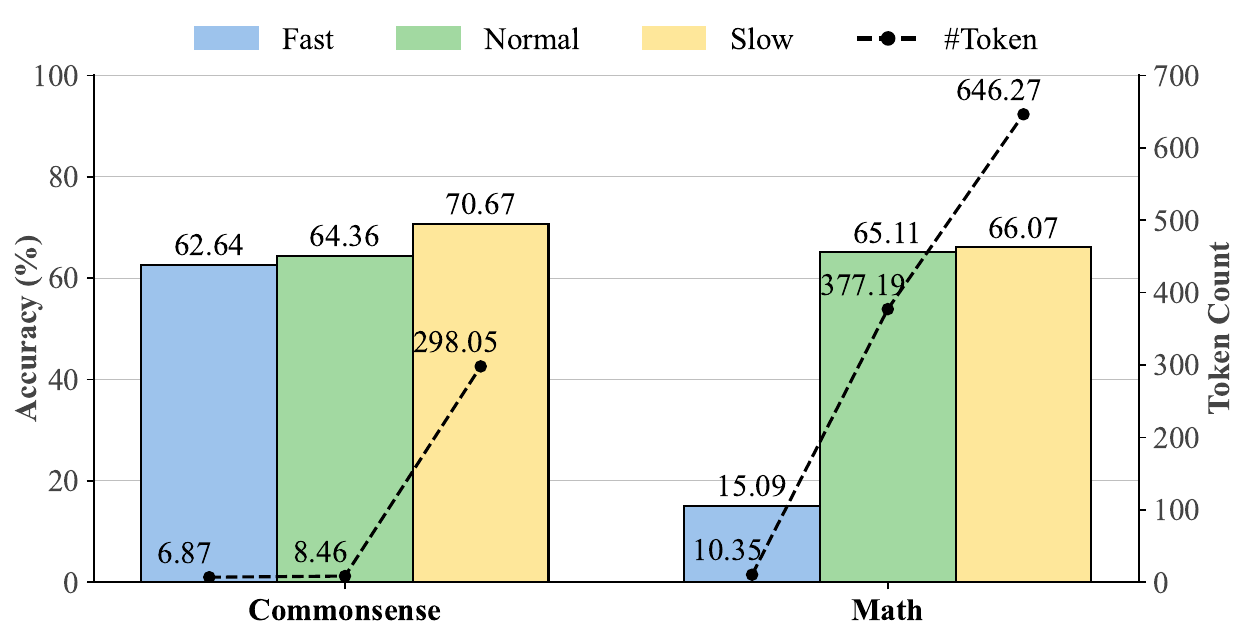}
\caption{Llama model's average accuracy and token consumption.}
\label{fig: llama_cost}
\end{minipage}  \hfill
\begin{minipage}[b]{0.36\linewidth}
\centering
\includegraphics[width=\linewidth]{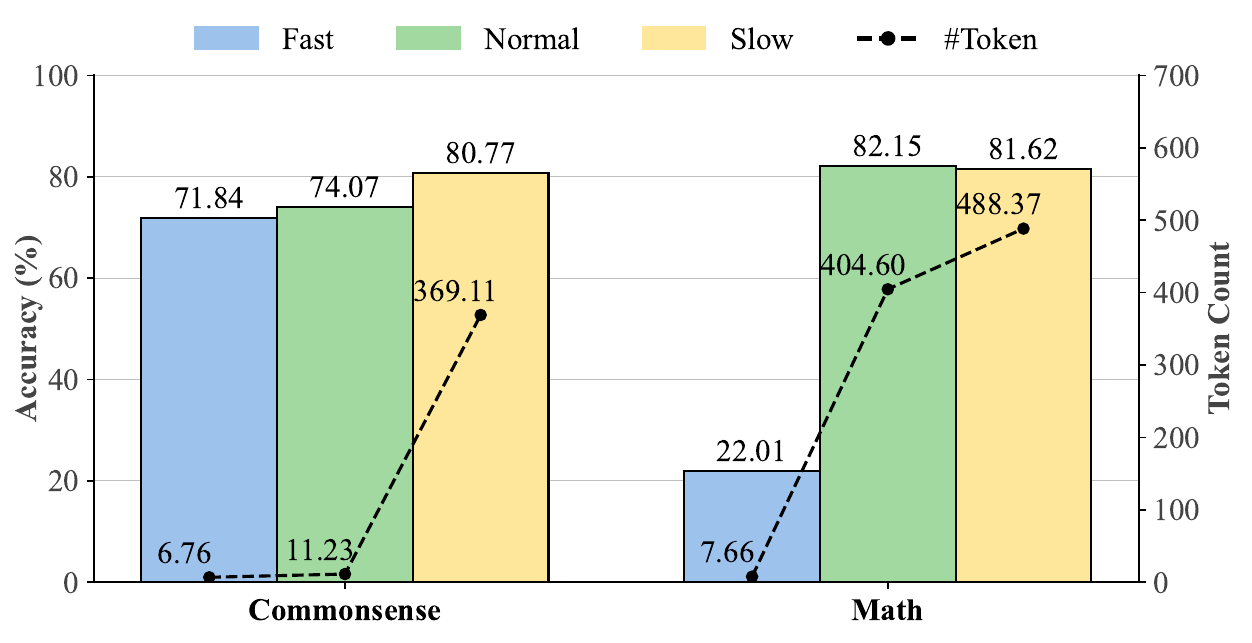}
\caption{Qwen model's average accuracy and token consumption consumption.}
\label{fig: qwen_cost}
\end{minipage} \hfill
\begin{minipage}[b]{0.25\linewidth}
\centering
\includegraphics[width=\linewidth]{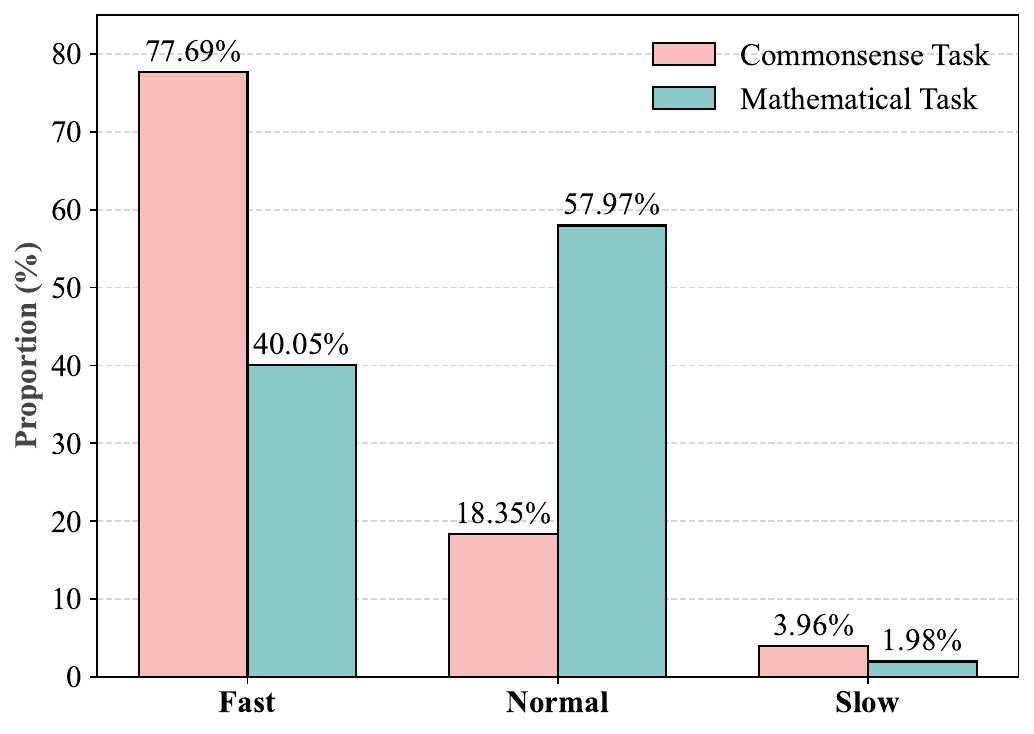}
\caption{The distribustion of TMC dataset.}
\label{fig: tmc_dis}
\end{minipage} 
\end{figure*}

\subsection{DynamicMind Framework}
The DynamicMind framework, illustrated in Figure \ref{fig:MD_framework}, introduces a cognitive control system that integrates prompt engineering with dynamic routing. It comprises two core components: (1) the \textit{LLM Thinker} with tri-mode thinking capabilities, and (2) the \textit{Mind Router} for adaptive mode selection. 

Emulating human cognition, the LLM Thinker operates in three distinct modes:  
\begin{itemize}
    \item \textbf{Fast Mode ($M_f$)}: Delivers rapid, intuitive responses by limiting cognitive depth, prioritizing speed over thorough reasoning.  
      
    \item \textbf{Slow Mode ($M_s$)}: Executes deep analytical reasoning for high-quality outputs with increased computational costs.  

    \item \textbf{Normal Mode ($M_n$)}: Leverages the LLM's native capabilities to balance response quality and efficiency.
\end{itemize}
Table \ref{tab:mode_spec} 
summarizes operational characteristics of the three modes. 
For an input question $q$, the Mind Router dynamically selects a suitable mode $m \in \{M_f, M_s, M_n\}$. Then the LLM Thinker gives an answer to $q$ in the selected mode $m$ within the zero-shot inference context, ensuring task-appropriate reasoning.

\subsection{Tri-Mode Thinking System}
The three modes $\mathcal{M} = \{M_f, M_n, M_s\}$ are achieved through distinct text templates $\mathcal{T}_m$ that shape the reasoning behavior of the LLM Thinker 
via explicit instructions. Specifically, for a question $q$, each mode $m \in \mathcal{M}$ generates an output sequence $y = (y_1, \ldots, y_T)$ via the following autoregressive process:  
\begin{equation}
P_m(y|q) = \prod_{t=1}^T P\left(y_t \mid \! \underbrace{\mathcal{T}_m}_{\text{\tiny Mode Prompt}} \! \oplus q\oplus y_{<t}; \theta\right),
\end{equation}
where $P_m(y|q)$ is the probability distribution over responses $y$ given $q$ under mode $m$, $\mathcal{T}_m$ denotes the prompt template encoding the reasoning strategy for mode $m$, $\oplus$ denotes the concatenation  operation, $y_{<t} = (y_1, \ldots, y_{t-1})$ are tokens generated before timestep $t$, and $\theta$ denotes the frozen parameters of the LLM Thinker.

To be specific, the three thinking modes are specifically designed as follows.
        
\begin{figure*}[h!t] 

\end{figure*}

\paragraph{Fast Mode: Quick and Intuitive Thinking.}
The fast mode prompts the LLM Thinker to provide direct answers without intermediate reasoning, similar to human System 1 thinking. This reduces computational load by minimizing token consumption \citep{wei2022chain, kojima2022large, NEURIPS2024_e304e04a}. The fast mode prompt template $\mathcal{T}_{\text{fast}}$ is constructed with specific constraints to ensure direct responses:
\[
\mathcal{T}_{\text{fast}} = \mathcal{T}_{\text{intuition}} \oplus \mathcal{T}_{\text{effortless}}.
\]
Details of the full prompt are put in Appendix \ref{app: fast temp}.

\paragraph{Slow Mode: Long and Analytical Thinking.}
The slow mode enhances the performance for complex problems by extending reasoning steps, similar to human System 2 thinking \cite{wang2022self, yao2024tree, besta2024graph}. The prompt template $\mathcal{T}_{\text{slow}}$ includes: 1) $\mathcal{T}_{\text{Decomposition}}$: it systematically breaks down complex problems into minimal executable units, allowing the model to tackle each part individually.
2) $\mathcal{T}_{\text{Quality}}$: it ensures that reasoning steps are high-quality, concise, and supported by context, which helps maintain logical coherence and depth.
3) $\mathcal{T}_{\text{Verification}}$: it identifies and corrects potential errors in reasoning steps, ensuring connectivity and accuracy throughout the process. That is, 
\[
\mathcal{T}_{\text{slow}} = \mathcal{T}_{\text{Decomposition}} \oplus  \mathcal{T}_{\text{Quality}} \oplus \mathcal{T}_{\text{Verification}}.
\]
This mode is particularly effective for complex tasks requiring detailed analysis and mitigates errors made by other faster modes. Details of full prompts are shown in Appendix \ref{app: slow temp}.

\paragraph{Normal Mode: Preserving Native Capabilities.}
The normal mode maintains the model's original thinking paradigm from pretraining, serving as a baseline. The prompt consists of minimal task-specific instructions:
\[
\mathcal{T}_{\text{normal}} = \mathcal{T}_{\text{instruction}}.
\]
This mode preserves refined thinking patterns, enables fair comparisons, and ensures compatibility with other techniques without interference from the fast and slow modes.

\paragraph{Distinct Capabilities of Individual Thinking Modes.}
To explore the different abilities of individual LLM Thinking modes, we evaluate their ZSQA performance by accuracy and measure their efficiency by output token consumption. To be more representative, we perform comparison on both commonsense QA and mathematical reasoning.

Figures \ref{fig: llama_cost} and \ref{fig: qwen_cost} illustrate the average accuracy and token consumption for three thinking modes using Llama3.1 \citep{dubey2024llama} and Qwen2.5 \citep{qwen2.5} as base LLMs. Key observations include:\\ 
\noindent 1) No single mode $m_{opt} \in \mathcal{M}$ consistently outperforms others in both effectiveness and token consumption across all tasks.\\
\noindent 2) Fast mode is highly efficient with minimal tokens and achieves reasonable accuracy on commonsense tasks, but struggles with complex math reasoning, indicating strength in knowledge retrieval over intensive computation.\\
\noindent 3) Slow mode achieves the highest overall accuracy but incurs high token consumption, excelling in depth and difficult problems but sacrificing efficiency for simpler questions. Besides, in math QA, we notice it is more easier to fail on simpler questions with redundant thinking steps (i.e., overthinking phenomenon revealed in \cite{chen2024not, sui2025stop}).\\
\noindent 4) As the reflection of the model's intristic ability from pretraining, normal mode strikes a balance, enhancing commonsense task accuracy over fast mode with minimal token increase and performing nearly as well as slow mode on mathematical tasks with better efficiency. This demonstrate a potential trade-off between reasoning depth and token consumption efficiency has been learned by the model pretraining. However, such trade-off ability is not explicit, as well as not good enough.  


Inspired by these observations, we then propose a routing-based approach to integrate the united pros of all these three thinking modes.

\subsection{Thinking Density and TMC Dataset}

In this section, we introduce a new metric called \textit{Thinking Density}, which evaluates the capability of diverse thinking modes of the LLM Thinker. Based on this metric, we construct the Thinking Mode Capacity (TMC) dataset from existing data, which is the first dataset revealing the mapping between questions and their preferred thinking patterns.

\paragraph{Thinking Density.} Given a question $q$ and the current thinking mode $m \in \mathcal{M}$ of the LLM Thinker, where $\mathcal{M} = \{M_f, M_n, M_s\}$ (representing fast, normal, and slow modes, respectively), the \textit{Thinking Density} with $k$ random runs is defined as:
\begin{equation}
\label{thinkdensity}
    E_m^k(q) = \frac{\text{accuracy}^k_m(q)}{(\text{avg.tok}^k_m(q))^\alpha},
\end{equation}
where $\text{accuracy}^k_m(q)$ is the accuracy of the model, defined as the number of times the model generates the correct answer divided by $k$. $\text{avg.tok}^k_m(q)$ is the average token consumption in the $k$ responses of the LLM Thinker in mode $m$. The hyperparameter $\alpha$ governs the trade-off between accuracy and efficiency by balancing the importance of correct answers against the computational cost of generating responses.
\paragraph{The TMC Dataset.}
Based on the \textit{Thinking Density} metric,
we construct the TMC dataset which maps questions to their preferred thinking mode. Specifically, we extract total 39K questions 
from the training split of in-domain datasets we used in Section \ref{exp:datasets}. For each question $q$, we identify the most effective thinking mode, denoted by $m^*_q$, as
\begin{equation}
m^*_q = \arg\max_{m \in \mathcal{M}} E^k_m(q)\label{equ_m_q}.
\end{equation}
Eq. \eqref{equ_m_q} selects the thinking mode $m$ that maximizes the Thinking Density $E^k_m(q)$ for the given question $q$. By coupling each question $q$ with its corresponding suitable thinking mode $m^*_q$, we set $k=10$ and $\alpha = 1$ to create the TMC dataset, denoted as $\mathcal{D}_{\text{TMC}} = \{(q_i, m_i^*)\}$.

The TMC dataset is a valuable resource for understanding the relationship between different types of questions and their suitable thinking modes. The statistics of label distribution for specific question types can reveal task-level thinking patterns. As shown in Figure \ref{fig: tmc_dis}.(c), there does not exist any single mode dominant the others upon all questions, highlighting the necessity of thinking mode selection. Besides, the mathematical tasks tends to utilize the Normal and Slow Mode rather than the Fast Mode, which is consistent with the fact that mathematical problems are typically thought-intensive, requiring multiple reasoning step \citep{boaler2022mathematical,guan2025rstarmath}. In contrast, the commonsense tasks shows a dominant preference for the Fast Mode. This suggests that commonsense questions are more inclined towards fast and intuitive answering. This observation can analogy to the findings in \citep{turk2009neural,nelli2023neural} that reveal that the human brain can handle commonsense knowledge via fast knowledge assembly and efficient replay in a short neural circuit. Details of the TMC dataset construction method are in Appendix \ref{app: tmc construction details}.
\subsection{Mind Router}
\label{mind_router}
The Mind Router constitutes the core decision-making module that dynamically selects a proper thinking mode $R_q$ for each question $q$. Building on our Thinking Density metric $E^k_m(q)$, we formalize the routing strategy through a dual-objective optimization framework that achieves Pareto optimality in accuracy-efficiency tradeoffs.

\begin{table*}[!h]
    \centering
    \caption{Performance comparison between DynamicMind and baselines across in-domain and out-of-domain tasks.  \textit{ACC}, \textit{\#Token}, and \textit{TD} refer to accuracy (\%), token consumption, and thinking density, respectively.} 
    \setlength{\tabcolsep}{6pt}
    \vskip -0.10in
    \resizebox{\linewidth}{!}{
    \begin{tabular}{lccccccccccccc}
    \toprule
     & & \multicolumn{4}{c}{\textbf{In-domain tasks}} & & \multicolumn{4}{c}{\textbf{Out-of-domain tasks}} & & \multicolumn{2}{c}{\textbf{Average}} \\
    \cmidrule(lr){3-6} \cmidrule(lr){8-11} \cmidrule(lr){13-14}
    \textbf{Method} & & \multicolumn{2}{c}{Math} & \multicolumn{2}{c}{CommonsenseQA} & & \multicolumn{2}{c}{MMLU} & \multicolumn{2}{c}{ScienceQA} & &  &  \\
    & & ACC$\uparrow$(\#Token$\downarrow$) & TD$\uparrow$ & ACC$\uparrow$(\#Token$\downarrow$) & TD$\uparrow$ & & ACC$\uparrow$(\#Token$\downarrow$) & TD$\uparrow$ & ACC$\uparrow$(\#Token$\downarrow$) & TD$\uparrow$ & & ACC$\uparrow$(\#Token$\downarrow$) & TD$\uparrow$ \\
    \midrule
    \multicolumn{14}{c}{\textit{Llama-3.1-Instruct-8B}} \\
    \midrule
    CoT & & \textbf{66.28}(425.74) & {0.16} 
        & 62.99(184.92) & 0.34
        & & {52.02}(290.51) & 0.18
        & 66.23(170.64) & 0.39
        & & 61.88(267.95) & 0.27 \\
    PBC & & 65.07(\underline{377.69}) & \underline{0.17} 
        & \textbf{69.98}(188.99) & \underline{0.37}
        & & \textbf{56.74}(\underline{251.19}) & \underline{0.23}
        & \textbf{73.51}(194.68) & 0.38
        & & \textbf{66.33}(253.14) & 0.29 \\
    TALE-EP & & \underline{65.15}(668.13) & 0.10 
        & 62.47(\underline{178.22}) & 0.35
        & & 50.81(350.51) & 0.14
        & 68.01(\underline{169.02}) & \underline{0.40}
        & & 61.61(341.47) & 0.25 \\
    DynamicMind
    & & 62.65(\textbf{264.37}) & \textbf{0.24}
        & \underline{63.91}(\textbf{30.32}) & \textbf{2.11}
        & & \underline{52.48}(\textbf{34.17}) & \textbf{1.54}
        & \underline{70.36}(\textbf{49.32}) & \textbf{1.43}
        & & \underline{62.35}(\textbf{94.55}) & \textbf{1.33} \\
    \midrule
    \multicolumn{14}{c}{\textit{Qwen-2.5-Instruct-7B}} \\
    \midrule
    CoT & & \textbf{82.15}(433.02) & 0.19 
        & \textbf{80.64}(214.82) & 0.38
        & & \textbf{73.80}(314.86) & 0.23
        & \textbf{79.83}(217.98) & 0.37
        & & \textbf{79.11}(295.17) & 0.29 \\
    PBC & & \underline{82.01}(403.13) & 0.20
        & 78.91(\underline{50.85}) & \underline{1.55}
        & & 70.58(202.12) & 0.35
        & 74.41(72.08) & 1.03
        & & 76.48(182.05) & 0.78 \\
    TALE-EP & & 64.25(\textbf{179.94}) & \textbf{0.36}
        & 73.16(\textbf{46.06}) & \textbf{1.59}
        & & 71.25(\underline{67.23}) & \underline{1.06}
        & {76.49}(\textbf{38.52}) & \textbf{1.99}
        & & 71.29(\textbf{82.94}) & \textbf{1.25} \\
    DynamicMind
    & & 81.54(\underline{394.24}) & \underline{0.21} 
        & \underline{79.41}(51.84) & 1.53
        & & \underline{72.12}(\textbf{49.04}) & \textbf{1.47}
        & \underline{77.10}(\underline{55.16}) & \underline{1.40}
        & & \underline{77.54}(\underline{137.57}) & \underline{1.15} \\
    \bottomrule
    \end{tabular}
    }
    \label{tab:main_results}
    \vspace{-5pt}
\end{table*}

\begin{table*}[!h]
    \centering
    
    \caption{Performance comparison of the DynamicMind framework with MindRouter versus single thinking modes. \textit{ACC}, \textit{\#Token}, and \textit{TD} refer to accuracy (\%), token consumption, and thinking density, respectively.} 
   
    \setlength{\tabcolsep}{6pt}
    \vskip -0.10in
    \resizebox{\linewidth}{!}{
    \begin{tabular}{lccccccccccccc}
    \toprule
     & & \multicolumn{4}{c}{\textbf{In-domain tasks}} & & \multicolumn{4}{c}{\textbf{Out-of-domain tasks}} & & \multicolumn{2}{c}{\textbf{Average}} \\
    \cmidrule(lr){3-6} \cmidrule(lr){8-11} \cmidrule(lr){13-14}
    \textbf{Method} & & \multicolumn{2}{c}{Math} & \multicolumn{2}{c}{CommonsenseQA} & & \multicolumn{2}{c}{MMLU} & \multicolumn{2}{c}{ScienceQA} & &  &  \\
    & & ACC$\uparrow$(\#Token$\downarrow$) & TD$\uparrow$ & ACC$\uparrow$(\#Token$\downarrow$) & TD$\uparrow$ & & ACC$\uparrow$(\#Token$\downarrow$) & TD$\uparrow$ & ACC$\uparrow$(\#Token$\downarrow$) & TD$\uparrow$ & & ACC$\uparrow$(\#Token$\downarrow$) & TD$\uparrow$ \\
    \midrule
    \multicolumn{14}{c}{\textit{Llama-3.1-Instruct-8B}} \\
    \midrule
    DynamicMind \\
     \ \ \ \ \ \ \ \textit{Fast-Only} & & 15.09(\textbf{10.35}) & \textbf{1.46} 
    & 62.64(\textbf{6.87}) & \textbf{9.12}
    & & 51.85(\textbf{6.90}) & \textbf{7.52}
    & \underline{71.72}(\textbf{7.01}) & \textbf{10.23}
    & & 50.32(\textbf{7.78}) & \textbf{7.08} \\
     \ \ \ \ \ \ \ \textit{Normal-Only} & & \underline{65.11}(377.19) & 0.17 
    & \underline{64.36}(\underline{8.46}) & \underline{7.61}
    & & 49.32(35.70) & 1.38
    & 69.15(\underline{43.31}) & \underline{1.60}
    & & 61.99(116.16) & \underline{2.69} \\
     \ \ \ \ \ \ \ \textit{Slow-Only} & & \textbf{66.07}(646.27) & 0.10 
    & \textbf{70.67}(298.05) & 0.24
    & & 59.72(405.95) & 0.15
    & \textbf{73.99}(283.33) & 0.26
    & & \textbf{67.61}(408.40) & 0.19 \\
     \ \ \ \ \ \ \ \textit{MindRouter} & & 62.65(\underline{264.37}) & \underline{0.24}
        & {63.91}({30.32}) & {2.11}
        & & \underline{52.48}(\underline{34.17}) & \underline{1.54}
        & 70.36(49.32) & 1.43
        & & \underline{62.35}(\underline{94.55}) & 1.33 \\
    \midrule
    \multicolumn{14}{c}{\textit{Qwen-2.5-Instruct-7B}} \\
    \midrule
    DynamicMind \\
    
     \ \ \ \ \ \ \ \textit{Fast-Only} & & 22.01(\textbf{7.66}) & \textbf{2.87} 
    & 71.84(\textbf{6.76}) & \textbf{10.62}
    & & 69.50(\textbf{7.02}) & \textbf{9.90}
    & 72.23(\textbf{6.98}) & \textbf{10.35}
    & & 58.90(\textbf{7.11}) & \textbf{8.44} \\
     \ \ \ \ \ \ \ \textit{Normal-Only} & & \textbf{82.15}(404.60) & 0.20 
    & 74.07(\underline{11.23}) & \underline{6.60}
    & & \underline{72.70}(56.66) & 1.28
    & 76.29(63.65) & 1.20
    & & 76.30(\underline{134.04}) & \underline{2.32} \\
     \ \ \ \ \ \ \ \textit{Slow-Only} & & \underline{81.62}(488.37) & 0.17 
    & \textbf{80.77}(369.11) & 0.22
    & & \textbf{73.10}(481.42) & 0.15
    & \textbf{80.12}(361.85) & 0.22
    & & \textbf{78.90}(425.19) & 0.19 \\
    \ \ \ \ \ \ \ \textit{MindRouter} & & 81.54(\underline{394.24}) & \underline{0.21} 
        & \underline{79.41}(51.84) & 1.53
        & & 72.12(\underline{49.04}) & \underline{1.47}
        & \underline{77.10}(\underline{55.16}) & \underline{1.40}
        & & \underline{77.54}(137.57) & 1.15 \\
    \bottomrule
    \end{tabular}
    }  \label{tab:main_results_single_mode}
    \vspace{-10pt}
\end{table*}

\paragraph{Pareto Optimal Routing}
For each question $q$, we define
1) Accuracy objective: $\text{Acc}^k_m(q) = \log \text{(accuracy}^k_m(q))$. 2) Efficiency objective: $\text{Eff}^k_m(q) = -\log(\text{avg.tok}^k_m(q))$. Therefore, according to its definition (equation \ref{thinkdensity}), the logrithm form of the Thinking Density $\text{log}(E^k_m(q)) = \text{Acc}^k_m(q) + \alpha\text{Eff}^k_m(q)$. Both objectives become better when they are larger. With taking the logarithm does not change the relative size of numerical values, we have a routing strategy $R$ \textit{Pareto dominates} $R'$ if:
\begin{equation}
    \begin{cases}
    \mathbb{E}_q[\text{Acc}^k_{R_q}(q)] \geq \mathbb{E}_q[\text{Acc}^k_{R'_q}(q)] \\
    \mathbb{E}_q[\text{Eff}^k_{R_q}(q)] \geq \mathbb{E}_q[\text{Eff}^k_{R'_q}(q)],
    \end{cases}
\end{equation}
with strict inequality in at least one objective. The Pareto frontier contains all non-dominated strategies. Based on this definition, we demonstrate the pareto optimality of the dynamic routing beyond all the three single thinking modes by providing a detailed theorem with proof in Appendix \ref{app:proof}.

\paragraph{Router Training.}
We treat the Mind Router as a classification model $\text{MR}_\phi(q): \mathcal{Q}\mapsto\mathcal{M}$ and train it on the proposed TMC dataset $\mathcal{D}_{\text{TMC}} = \{(q_i, m_i^*)\}$, where $m_i^* = \arg\max_{m \in \mathcal{M}} E^k_m(q_i)$. The cross-entropy loss:
\begin{equation}
\mathcal{L}(\phi) = -\mathbb{E}_{(q,m^*)\sim\mathcal{D}_{\text{TMC}}}[\log p_\phi(m^*|q)]
\end{equation}
enables the router to approximate the theoretically optimal $R^*_q$. 

\section{Experiments}

In this section, we empirically evaluate the proposed DynamicMind framework.

\subsection{Experimental Setup}
\label{exp:setup}
\paragraph{Datasets.}
\label{exp:datasets}
We evaluate DynamicMind on both in-domain and out-of-domain datasets for ZSQA. In-domain datasets include GSM8K \citep{cobbe2021training} and MATH \citep{DBLP:conf/nips/HendrycksBKABTS21} for mathematical reasoning, as well as commonsense QA datasets such as BoolQ \citep{clark2019boolq}, PIQA \citep{bisk2020piqa}, SIQA \citep{sap2019socialiqa}, HellaSwag \citep{zellers2019hellaswag}, WinoGrande \citep{sakaguchi2021winogrande}, ARC-e, ARC-c \citep{clark2018think}, and OBQA \citep{mihaylov2018can}. Out-of-domain datasets, including MMLU \citep{DBLP:conf/iclr/HendrycksBBZMSS21} and ScienceQA \citep{lu2022learn}, present greater challenges with domain-specific questions spanning Biology, Physics, Chemistry, Medicine, and more. Dataset statistics are detailed in Appendix \ref{app: datasets statistics}.
\paragraph{Baselines and Metrics.}
We compare DynamicMind against three baselines: 1) Vanilla Chain-of-Thought (CoT) \citep{wei2022chain}, which uses step-by-step reasoning prompts; 2) Prototype-Based Clustering (PBC) \citep{reimers-2019-sentence-bert}, which selects modes based on cosine similarity with semantic centroids; and 3) TALE-EP \citep{han2024token}, which predicts token budgets to constrain reasoning length. For each question, the LLM Thinker is evaluated $k=3$ times, with the average accuracy and response token consumption reported across trials. Implementation details of all methods are provided in Appendix \ref{app: imp details}.
\subsection{Main Results}
Table \ref{tab:main_results} demonstrates the effectiveness (accuracy) and efficiency (token consumption) comparison among DynamicMind and baselines. As can be seen, with MindRouter,
DynamicMind exhibits a consistent ability to outperform baseline methods by striking a thoughtful balance between accuracy and efficiency. When use Llama as the base model, it achieves strong performance across CommonsenseQA (63.91\%), MMLU (52.48\%), and ScienceQA (70.36\%) while consuming fewer tokens, resulting in TDs of 2.11, 1.54, and 1.43, respectively, with an average TD of 1.33—significantly surpassing CoT (TD 0.27), PBC (TD 0.29), and TALE-EP (TD 0.25). Similarly, when use Qwen, it achieves higher accuracy in ScienceQA (77.1\% vs. PBC's 74.41\% and TALE-EP's 76.49\%) while maintaining efficient token consumption, showing an exceptional accuracy-cost trade-off across tasks.

To illustrate the role of MindRouter, we include metrics for individual thinking modes (i.e., Fast-Only, Normal-Only, and Slow-Only) in Table \ref{tab:main_results_single_mode}. The results reveal that Fast Mode minimizes token consumption but suffers from limited accuracy, particularly in precision-critical tasks like Math. In contrast, Normal and Slow Modes improve accuracy but incur significantly higher token consumption. DynamicMind, equipped with MindRouter, effectively coordinates these thinking modes to achieve an optimal balance—delivering accuracy comparable to the most accurate mode (e.g., Slow Mode) while drastically reducing token consumption. This adaptive trade-off makes DynamicMind provide proper performances to balance accuracy and efficiency. 

\subsection{MindRouter Transferability for LLMs}

Table \ref{tab:transfer} summarizes the transferability of Mind Routers across LLMs. When Llama used Qwen's router, token consumption generally decreased across tasks, but accuracy dropped for most tasks except for a slight improvement in mathematical reasoning. Conversely, when Qwen adopted Llama's router, token consumption varied, with CommonsenseQA showing a notable increase ($\text{+115}$ tokens), while accuracy declined across all tasks. On average, cross-model router transfer resulted in a ~2\% accuracy drop and token fluctuations of ±30 tokens. These results indicate that while Mind Routers exhibit a degree of transferability, differences in reasoning capabilities between LLMs limit their effectiveness when transferred.
\begin{table}[h!]
    \centering
    \caption{{Performance of the Mind Router transferred from one LLM to another.}} 
    \setlength{\tabcolsep}{6pt}
    \resizebox{\linewidth}{!}{
    \begin{tabular}{lccccc}
    \toprule
    \textbf{Metric} & Math & CommonsenseQA & MMLU & ScienceQA & \textbf{Avg.} \\
    \midrule
    \multicolumn{6}{c}{\textit{Transfer Qwen's MindRouter to Llama}} \\
    \midrule
    $\Delta \text{ACC} \uparrow$ & +0.71 & -5.68 & -1.26 & -0.50 & -1.68 \\
    $\Delta \text{\#Token} \downarrow$ & -10.57 & -91.79 & -17.25 & +9.07 & -27.64 \\
    \midrule
    \multicolumn{6}{c}{\textit{Transfer  Llama's MindRouter to  Qwen}} \\
    \midrule
    $\Delta \text{ACC} \uparrow$ & -1.62 & -5.92 & -0.26 & -3.32 & -2.78 \\
    $\Delta \text{\#Token} \downarrow$ & -10.99 & +115.00 & +17.95 & -20.46 & +25.38 \\

    \bottomrule
\end{tabular}

}
\vspace{-6pt}
\label{tab:transfer}
\end{table}
\begin{figure*}[h]
    \centering
\includegraphics[width=1\textwidth]{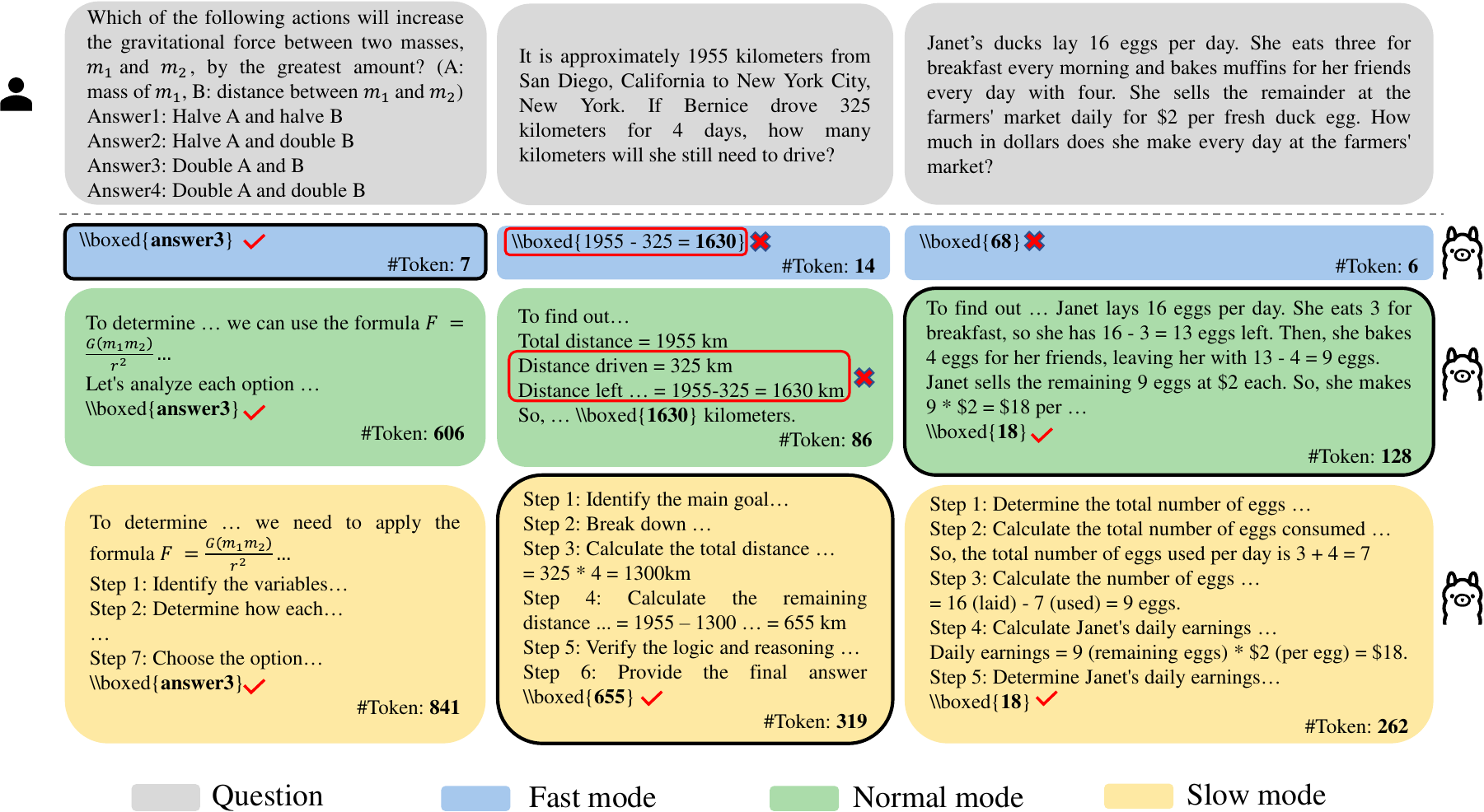}
\caption{Illustrations of  the cases where the Mind Router effectively selects the suitable thinking mode.}
\label{fig:case}
\end{figure*}
\subsection{Ablation Study}
\paragraph{Necessity of normal mode.}
\begin{table}[h!t]
    \centering
    \caption{{The results with ablating the normal mode from the DynamicMind Framework.}} 
    \setlength{\tabcolsep}{6pt}
    \resizebox{\linewidth}{!}{
    \begin{tabular}{lccccc}
    \toprule
    \textbf{Metric} & Math & CommonsenseQA & MMLU & ScienceQA & \textbf{Avg.} \\
    \midrule
    \multicolumn{6}{c}{\textit{Llama-3.1-Instruct-8B w/o normal mode}} \\
    \midrule
    $\Delta \text{ACC} \uparrow$ & +1.74 & -6.45 & +7.19 & +2.12 & +1.15 \\
    $\Delta \text{\#Token} \downarrow$ & +251.12 & -127.97 & +304.74 & +254.77 & +170.67 \\
    \midrule
    \multicolumn{6}{c}{\textit{Qwen-2.5-Instruct-7B w/o normal mode}} \\
    \midrule
    $\Delta \text{ACC} \uparrow$ & -32.74 & -0.57 & -1.88 & -4.87 & -10.02 \\
    $\Delta \text{\#Token} \downarrow$ & -143.08 & -4.75 & +7.39 & -48.18 & -47.16 \\

    \bottomrule
\end{tabular}
}
\label{tab:normal ablation}
\end{table}
The results in Table \ref{tab:normal ablation} underscore the necessity of the normal mode. When the normal mode is removed, Llama's accuracy improves marginally, but this comes with a substantial increase in token consumption, indicating inefficiency. For Qwen, although token consumption slightly decreases, accuracy suffers a significant drop, highlighting a trade-off that undermines performance. These outcomes illustrate that the normal mode is crucial for enhancing LLM thinking efficiency, as it effectively balances computational cost and performance accuracy. Overall, the inclusion of the normal mode provides clear benefits, reinforcing its role as a vital component alongside fast and slow thinking paradigms.

\paragraph{Sensitivity on $\alpha$.}
\vspace{-8pt}
\begin{table}[h!t]
    \centering
    \caption{The results with varying the hyperparameter $\alpha$ in the DynamicMind Framework.}
    \label{tab:alpha_ablation}
    \setlength{\tabcolsep}{6pt}
    \resizebox{\linewidth}{!}{
        \begin{tabular}{lccccccc}
            \toprule
            \textbf{Metric} & & {$\alpha$} & Math & CommonsenseQA & MMLU & ScienceQA & \textbf{Avg.} \\
            \midrule
            \multicolumn{8}{c}{\textit{Llama-3.1-Instruct-8B}} \\
            \midrule
            
            \multirow{2}{*}{$\Delta \text{ACC} \uparrow$}
            & & 0 & +1.27 & +4.62 & -1.42 & +0.73 & +1.30  \\
            & & 2 & +0.11 & +0.05 & -1.00 & -1.82 & -0.66 \\

            \midrule
            
            \multirow{2}{*}{$\Delta \text{\#Token} \downarrow$}
            & & 0 & +65.53 & +95.6 & +30.13 & +21.52 & +53.20  \\
            & & 2 & +4.37 & -2.51 & -3.88 & +7.20 & +1.30 \\
            \midrule
            \multicolumn{8}{c}{\textit{Qwen-2.5-Instruct-7B}} \\
            \midrule
            
            \multirow{2}{*}{$\Delta \text{ACC} \uparrow$}
            & & 0 & +0.11 & +0.05 & -1.00 & -1.82 & -0.66 \\
            & & 2 & +0.29 & -0.13 & -0.24 & -1.37 & -0.36 \\
            \midrule
            
            \multirow{2}{*}{$\Delta \text{\#Token} \downarrow$}
            & & 0 & +4.37 & -2.51 & -3.88 & +7.20 & +1.30 \\
            & & 2 & +4.61 & -0.94 & +1.08 & -4.02 & +0.18 \\
            \bottomrule
        \end{tabular}
    }
\end{table}

The $\alpha$ parameter governs the trade-off between accuracy and computational efficiency defined in Eq. \ref{thinkdensity}. Table \ref{tab: alpha ablation} demonstrates how varying $\alpha$ values influence DynamicMind's performance and token consumption. For instance, setting $\alpha$ to 0 improved accuracy but resulted in increased token consumption for both models. Conversely, higher $\alpha$ values, such as $\alpha=2$, reduced token consumption with minimal impact on accuracy. Thus, these findings underscore the role of $\alpha$ in balancing performance and operational overhead by modulating the distribution of fast, normal, and slow thinking modes.

\subsection{Case Study}
We demonstrate the cases where the Mind Router effectively selects the suitable mode in Figure \ref{fig:case}.

1) Left Case: Fast Mode. In this commonsense question about Newton's law, the Llama model answered correctly across all modes. The fast mode, however, used significantly fewer tokens, saving 599 tokens compared to the normal mode and 834 tokens compared to the slow mode. For 6\% of tasks, the fast mode matched or exceeded the normal mode's accuracy, with an average token savings of 120.52 and a 33.89\% accuracy improvement. This suggests the fast mode is efficient for simpler tasks due to the Llama model's strong zero-shot capabilities.

2) Right Case: Normal Mode. This mathematics problem illustrates the normal mode's strength, where it succeeded without explicit instructions, unlike the fast mode. It saved 134 tokens compared to the slow mode. On average, 14.13\% of tasks saw a 49.95\% accuracy improvement over the fast mode, and 22.33\% maintained accuracy similar to the slow mode while saving 288.45 tokens. This mode balances efficiency and reasoning for moderately complex problems.

3) Middle Case: Slow Mode. In another math problem, both fast and normal modes made errors, while the slow mode succeeded by incorporating crucial temporal information. Despite consuming 233 more tokens than the normal mode, it improved accuracy for 3.11\% of tasks by 53.84\%, at an additional cost of 277.30 tokens. The slow mode excels in complex problems needing detailed attention, enhancing reasoning performance despite higher resource use.

\section{Conclusion}
We introduce DynamicMind, a tri-mode reasoning framework that extends the dual-process paradigm by adding a normal think mode, enabling LLMs to preserve native reasoning while adaptively selecting suitable modes via the Mind Router. Validated across LLMs and domains, DynamicMind enhances performance and reduces computational overhead, with the Thinking Density metric and TMC dataset further advancing adaptive reasoning research.

\section*{Limitations}

Though the proposed DynamicMind has been illustrated as effective, it faces limitations including: 1) First, its balance between computational cost and accuracy may falter in scenarios where either efficiency or cost is exclusively prioritized, reducing the benefits of its complementary modes. 2) Second, the integration of multiple modes introduces slight overhead, which may hinder its use in resource-constrained environments. Future work should focus on refining those aspects to enhance the adaptability and efficiency.

\section*{Ethic Statement}
There is no ethical problem in our study.

\bibliography{custom}
\newpage
\appendix
\onecolumn

\section{Thinking Modes Templates.}
The fast, normal and slow thinking modes in DynamicMind are implemented by controlling the LLM's system prompt. In this section, we provide our carefully designed prompt template for reference.
\subsection{Fast Mode}
\label{app: fast temp}


\begin{prompt}
{\textit{Fast Mode System Prompt}}{template}
\label{exmp_template}
Engage in your \textbf{System 1 Thinking Mode}. You MUST react based on the following rules:

    \quad\quad 1. Respond immediately with your first thought, based purely on gut feeling or your own knowledge.
    
    \quad\quad 2. No reasoning or explanations allowed.
    
    \quad\quad 3. No thinking process needed.
\end{prompt}

\subsection{Normal Mode}
\begin{prompt}
{\textit{Normal Mode System Prompt}}{template}
\label{exmp_template}
You are an AI assistant who provides helpful responses.
\end{prompt}

\subsection{Slow Mode}
\label{app: slow temp}
\begin{prompt}
{\textit{Slow Mode System Prompt}}{template}
\label{exmp_template}
    Engage in your \textbf{System 2 Thinking Mode}. You MUST think step by step based on the following rules:
        
    \quad 1. Problem Decomposition:
    
      \quad\quad - First identify the main goal/question
      
      \quad\quad - Break down into smallest helpful executable units
      
      \quad\quad - Each step focuses on ONE specific sub-problem
    
      \quad\quad - Ensure clear logical flow between steps
      
      \quad\quad - Steps should build towards final solution

    \quad 2. Step Quality and Execution:
    
      \quad\quad - Keep each step focused and concise
      
      \quad\quad - Present core logic and detailed explanations (1-3 sentences) for each step
      
      \quad\quad - Support reasoning with specific context and supporting details
      
      \quad\quad - Continuously evaluate if steps can be broken down further

    \quad 3. Step Verification:
      
      \quad\quad - After each step, verify the logic and reasoning
      
      \quad\quad - Ensure logical consistency and correctness of explanations
      
      \quad\quad - Check if each step effectively contributes to the overall goal
      
      \quad\quad - Address any discrepancies or errors in subsequent steps
      
      \quad\quad - Maintain a smooth transition between steps

    Always remember: Many precise, focused steps > Few broad steps.
\end{prompt}

\section{Details of the TMC Dataset Construction}
\label{app: tmc construction details}
We designed a dataset construction method to train a high-quality Mind Router. Specifically, we first obtain raw data by evaluating the LLM $k$ times on each task data using the fast, normal, and slow thinking modes, respectively. From this raw data, we then filter out task data where the LLM's accuracy across the three thinking modes is consistently below 80\%, as these tasks are considered too difficult for the model. Subsequently, we remove tasks where the average response lengths of the LLM in fast, normal, and slow modes do not satisfy the condition that the response length in fast mode is less than or equal to the length in normal mode, and the length in normal mode is less than or equal to the length in slow mode. This step ensures adherence to the design principles of the three modes. Then, task data with lengths exceeding the maximum sequence length of the MindRouter tokenizer are deleted to ensure that data truncation is not required during MindRouter training. Finally, the think density is calculated using the specified alpha value, and the mode with the highest think density is identified as the optimal mode for the LLM to complete that task.

\section{Proof of Theorem}
\label{app:proof}
We here provide the theorem with proof to demonstrate the pareto optimality of ideal dynamic routing. Because of logrithm operation does not influence the relative numerical value, to make the proof concise, we use the logrithm form of thinking density, i.e., $\hat{E}_m^k(q) = \text{log}(E^k_m(q))$ to complete this proof.

\begin{theorem}[Dynamic Routing Pareto Optimality]
For any question distribution $\mathcal{D}_q$ and tradeoff parameter $\alpha>0$, the optimal router:
\begin{equation}
R^*_q = \arg\max_{m\in\mathcal{M}} \underbrace{\text{Acc}^k_m(q)}_{\text{Accuracy}} + \alpha \underbrace{\text{Eff}^k_m(q)}_{\text{Efficiency}}
\end{equation}
resides on the Pareto frontier and satisfies:
\begin{equation}
\forall m\in\mathcal{M},\ \begin{cases}
\mathbb{E}_q[\text{Acc}^k_{R^*_q}(q)] \geq \mathbb{E}_q[\text{Acc}^k_m(q)]\\
\mathbb{E}_q[\text{Eff}^k_{R^*_q}(q)] \geq \mathbb{E}_q[\text{Eff}^k_m(q)]
\end{cases}
\end{equation}
with strict inequality when $m$ is suboptimal for any $q\in\text{supp}(\mathcal{D}_q)$.
\end{theorem}
\begin{proof}
Given $\hat{E}^k_m(q) = \text{Acc}^k_m(q) + \alpha\text{Eff}^k_m(q)$. For any fixed mode $m$:
\begin{align*}
\mathbb{E}_q[\hat{E}^k_{R^*_q}(q)] &= \mathbb{E}_q\left[\max_{m'\in\mathcal{M}} \hat{E}^k_{m'}(q)\right] \\
&\geq \mathbb{E}_q[\hat{E}^k_m(q)] \quad \text{(pointwise optimality)} \\
&= \mathbb{E}_q[\text{Acc}^k_m(q)] + \alpha\mathbb{E}_q[\text{Eff}^k_m(q)]
\end{align*}
Rearranging terms, we have $\mathbb{E}_q[\text{Acc}_{R^*_q}^k(q)] - \mathbb{E}_q[\text{Acc}^k_m(q)]  
 \geq \alpha\left(\mathbb{E}_q[\text{Eff}^k_m(q)] - \mathbb{E}_q[\text{Eff}_{R^*_q}^k(q)]\right)$. If $\mathbb{E}_q[\text{Eff}^k_{R^*_q}(q)] < \mathbb{E}_q[\text{Eff}^k_m(q)]$, the RHS becomes positive, forcing $\mathbb{E}_q[\text{Acc}^k_{R^*_q}(q)] > \mathbb{E}[\text{Acc}^k_m(q)]$. Otherwise $\mathbb{E}_q[\text{Eff}^k_{R^*_q}(q)] \geq \mathbb{E}_q[\text{Eff}^k_m(q)]$ directly holds. Thus $R^*_q$ either strictly improves accuracy while matching efficiency, or maintains accuracy while strictly improving efficiency. This establishes Pareto dominance over any fixed mode $m$.
\end{proof}

\section{Datasets Infomations}
\label{app: datasets statistics}
In this section, we introduce the dataset information used in our experiments. When constructing the TMC dataset, we incorporated mathematical tasks and CommonsenseQA tasks. Therefore, the datasets corresponding to these two task types include: GSM8K \citep{cobbe2021training} and MATH \citep{DBLP:conf/nips/HendrycksBKABTS21} for mathematical reasoning, as well as Commonsense QA datasets such as BoolQ \citep{clark2019boolq}, PIQA \citep{bisk2020piqa}, SIQA \citep{sap2019socialiqa}, HellaSwag \citep{zellers2019hellaswag}, WinoGrande \citep{sakaguchi2021winogrande}, ARC-e, ARC-c \citep{clark2018think}, and OBQA \citep{mihaylov2018can}. We used the training set portions of these datasets to construct the TMC dataset, and their test set portions served as in-domain datasets for evaluating the LLM's dynamic thinking capabilities. Additionally, we selected the test sets of various subtasks from two other datasets, MMLU \citep{DBLP:conf/iclr/HendrycksBBZMSS21} and ScienceQA \citep{lu2022learn}, as out-of-domain datasets to measure the generalization ability of the LLM's dynamic thinking capabilities.

\section{Implementation Details.}
\label{app: imp details}
In our experiments, we implement the LLM Thinker using two representative open-source models as backbone: Llama-3.1-8B-Instruct and Qwen-2.5-7B-Instruct, selected for their strong instruction-following capabilities. For the Mind Router, we use DeBERTaV3-base. All experiments are conducted on a single NVIDIA A100 GPU. For the LLM Thinker, we set the temperature to 0.6 and a top-p value of 0.9. Additionally, we impose a token generation limit of 128 tokens for Fast Mode, 2048 tokens for Normal Mode, and 4096 tokens for Slow Mode, with the LLM ceasing generation upon reaching these mode-specific token limits.
\end{document}